\documentclass[conference, 10pt]{IEEEtran}
\let\IEEEmaketitle=\maketitle
\renewcommand{\maketitle}{\begingroup\let\footnote=\thanks \IEEEmaketitle\endgroup}

\usepackage{tabularx,ragged2e}
\usepackage{booktabs}
\newcolumntype{L}{>{\centering\arraybackslash}X}

\usepackage{multicol}
\usepackage{multirow}

\usepackage{boldline}

\usepackage[]{graphicx}

\usepackage[export]{adjustbox}
\usepackage[caption=false,font=footnotesize]{subfig} 

\usepackage{cite}

\usepackage[hidelinks, bookmarks=false]{hyperref}

\usepackage{amsmath}
\usepackage{amsfonts}
\usepackage{amssymb}
\usepackage{mathtools}

\usepackage{amsthm}
\newtheorem{lem}{Lemma}

\usepackage{fixmath}

\usepackage{balance}

\usepackage{siunitx}

\usepackage{enumitem}

\usepackage{soul}

\usepackage{xcolor}

\usepackage{cmap}
\usepackage{mmap}

\usepackage{flafter}

\usepackage{algorithm}
\usepackage{algpseudocode}

\usepackage{url}

\DeclareMathOperator*{\argmax}{argmax} 

\algnewcommand\algorithmicforeach{\textbf{for each}}
\algdef{S}[FOR]{ForEach}[1]{\algorithmicforeach\ #1\ \algorithmicdo}
\algdef{SE}[DOWHILE]{Do}{doWhile}{\algorithmicdo}[1]{\algorithmicwhile\ #1}%
\algnewcommand{\LineComment}[1]{\Statex \(\///\) #1}

\IEEEoverridecommandlockouts

\begin{document}
\bstctlcite{IEEEexample:BSTcontrol}


\title{\Large{Sensitivity-Aware Mixed-Precision Quantization and Width Optimization of\\Deep Neural Networks Through Cluster-Based Tree-Structured Parzen Estimation}}

\author{\IEEEauthorblockN{Seyedarmin Azizi*, Mahdi Nazemi*\thanks{*Seyedarmin Azizi and Mahdi Nazemi contributed equally to this work.}, Arash Fayyazi, and Massoud Pedram}
\IEEEauthorblockA{Department of Electrical \& Computer Engineering, University of Southern California, Los Angeles, CA, USA\\
    \url{{seyedarm,mnazemi,fayyazi,pedram}@usc.edu}}
}
		
\maketitle
\begin{abstract}

As the complexity and computational demands of deep learning models rise, the need for effective optimization methods for neural network designs becomes paramount. This work introduces an innovative search mechanism for automatically selecting the best bit-width and layer-width for individual neural network layers. This leads to a marked enhancement in deep neural network efficiency. The search domain is strategically reduced by leveraging Hessian-based pruning, ensuring the removal of non-crucial parameters. Subsequently, we detail the development of surrogate models for favorable and unfavorable outcomes by employing a cluster-based tree-structured Parzen estimator. This strategy allows for a streamlined exploration of architectural possibilities and swift pinpointing of top-performing designs. Through rigorous testing on well-known datasets, our method proves its distinct advantage over existing methods. Compared to leading compression strategies, our approach records an impressive 20\% decrease in model size without compromising accuracy. Additionally, our method boasts a 12\(\times\) reduction in search time relative to the best search-focused strategies currently available. As a result, our proposed method represents a leap forward in neural network design optimization, paving the way for quick model design and implementation in settings with limited resources, thereby propelling the potential of scalable deep learning solutions.
\end{abstract}

\section{Introduction}
\label{sec:intro}
\noindent
Deep neural networks (DNNs) have emerged as highly powerful and versatile tools for tackling real-world problems across various domains, including computer vision \cite{DBLP:conf/cvpr/HeZRS16,DBLP:conf/bmvc/ZagoruykoK16,DBLP:conf/iclr/DosovitskiyB0WZ21,DBLP:conf/nips/TolstikhinHKBZU21}, natural language processing \cite{DBLP:conf/nips/VaswaniSPUJGKP17,DBLP:conf/naacl/DevlinCLT19,DBLP:journals/corr/abs-2302-13971}, signal processing \cite{yildirim2018arrhythmia, baraeinejad2022design}, and diagnostics \cite{wang2021deep,razmara2024fever}.
The remarkable performance of DNNs can be attributed to their capacity to learn intricate patterns and representations from large-scale data, involving millions or even billions of arithmetic operations and parameters.
However, this success comes at the expense of high compute cycles, memory footprint, I/O bandwidth, and energy consumption, resulting in an increased carbon footprint and limitations in deploying them on resource-constrained platforms.

To address these challenges, quantization and structured pruning techniques have emerged as promising strategies, offering the potential to mitigate the computational and memory burden of DNNs while maintaining satisfactory accuracy \cite{DBLP:journals/corr/abs-1805-06085,DBLP:journals/corr/ZhouNZWWZ16,DBLP:conf/iclr/0022KDSG17,DBLP:conf/iccv/LuoWL17}.
These techniques have proven effective in making DNNs more efficient, enabling their deployment on edge devices, and accelerating inference in cloud environments.
However, the successful application of quantization and structured pruning heavily relies on finding the optimal bit-width and layer-width for each layer of the DNN, which presents a formidable challenge due to the exponential growth of the search space as the number of layers increases.

The challenges in optimizing the bit-width for each layer of a DNN are further compounded by the inherent diversity in weight distributions across different layers (see Fig.~\ref{fig:weight-distribution} in addition to the sensitivity of a DNN's predictions to each layer's weights.
\begin{figure}[tb]
    \centering
    \includegraphics[width=0.9\columnwidth]{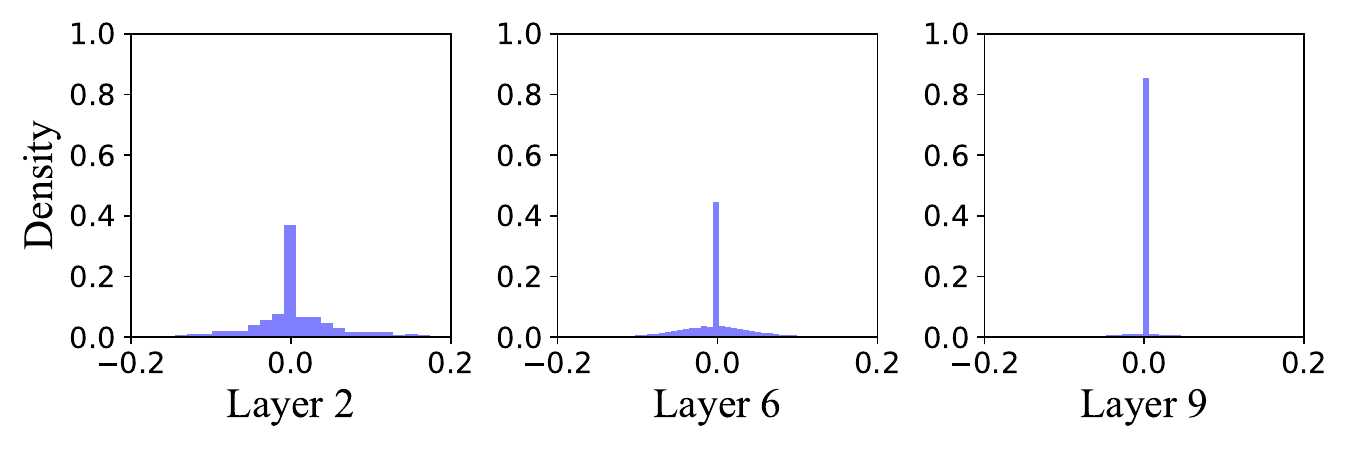}
    \caption[Weight distribution]{Distribution of weights in three representative layers of the MobileNetV1 architecture trained on the CIFAR-100 dataset.}
    \label{fig:weight-distribution}
\end{figure}
The weight distributions and sensitivity values can vary significantly, and as a result, different layers may benefit from different bit-widths to achieve maximum gains in terms of memory and computation cost savings without compromising accuracy.

Additionally, widening layers can profoundly impact the accuracy of the model \cite{DBLP:conf/bmvc/ZagoruykoK16}.
Therefore, the search process should not be confined to reducing the layer-width only; it must also identify situations where widening a layer, combined with a sufficiently reduced bit-width, leads to potential accuracy gains without compromising cost savings.

The search process for optimizing the bit-width and layer-width of DNNs should extend beyond traditional metrics such as FLOPs and memory footprint when evaluating cost savings.
While these metrics are useful, they may not directly reflect the real-world performance and efficiency of the model.
Therefore, it is imperative to consider factors like real latency, throughput, and energy consumption when comparing models in terms of cost savings.
Hence, devising an algorithm that effectively explores the bit-width and layer-width search space while considering the diverse weight distributions and sensitivity values and the varying impact of layer-width on accuracy and real-world performance becomes critical.
This necessitates an approach that can intelligently adapt the bit-width and layer-width for each layer, striking a balance between latency, throughput, or energy consumption and the preservation of model accuracy.

This paper presents a novel model-based optimization method based on tree-structured Parzen estimators (TPEs) \cite{DBLP:conf/nips/BergstraBBK11} to address the challenge of simultaneously searching for the optimal bit-width and layer-width values for DNN layers.
The major innovations of the presented optimization method are as follows:
\begin{itemize}
    \item{Exploiting Second-Order Derivatives: In addition to leveraging the distribution of layer weights, our method incorporates second-order derivatives of the loss function with respect to layer weights. This exponential pruning of the search space is particularly effective when dealing with pre-trained models, resulting in enhanced optimization efficiency.}
    \item{Hardware-Aware Objective Function: Our approach takes into account essential information about the target hardware that will execute the optimized DNN. We construct latency, throughput, and energy consumption models, which, combined with model accuracy, define a composite objective function guiding the optimization process.}
    \item{Handling Flat Loss Landscapes: DNNs often exhibit flat loss landscapes, which pose challenges for conventional optimizers. To address this, we fit surrogate distributions to both desirable and undesirable observations of the objective function. This adaptation enables us to achieve comparable or better objective values compared to state-of-the-art optimizers while significantly reducing convergence time, typically by a factor of at least 10x.}
    \item{Joint Optimization of Bit-Width and Layer-Width: Unlike conventional approaches that treat bit-width and layer-width independently, we search for joint optimal configurations. This novel approach enables us to discover configurations that yield higher-quality results, which would have otherwise been challenging to find with independent optimization.}
\end{itemize}
Through an extensive series of experiments, we demonstrate the effectiveness and efficiency of our method. The results showcase substantial improvements in DNN processing efficiency without sacrificing predictive accuracy. 

\section{Related Work}
\label{sec:prelim}
\noindent
Mixed-precision quantization exploits the fact that not all model parameters require the same level of precision to maintain model accuracy.
Quantization-aware training techniques simulate the quantization effects during the training process, allowing the model to adapt to lower precision, while post-training quantization techniques quantize DNN models after they have been trained at full precision.
The former typically achieves higher accuracy at the expense of increased training time.

In sensitivity-based mixed-precision quantization, first- or second-order gradient statistics, or other sensitivity metrics, are employed as a proxy to layer importance, which is subsequently used to assign an appropriate bit-width to each layer \cite{DBLP:conf/nips/DongYAGMK20,DBLP:conf/eccv/TangOWZJWZ22,DBLP:conf/ijcnn/Yuan0HP20}.
However, sensitivity-based quantization comes has some shortcomings.
First, it does not take into account the influence of quantized weights and input activations on output activations, which are inputs to other layers of a DNN including batch normalization layers commonly found after convolutional layers.
In other words, the batch statistics of batch normalization layers or weights of succeeding convolutional layers are trained with full-precision input activations, and any change in those input activations due to quantization may hamper model accuracy.
Second, gradient-based methods utilize gradients calculated on the full-precision model, thus overlooking the impact of quantization on gradients, which is a significant effect when using ultra-low bit-width quantization.
Third, the sensitivity values are often calculated based on the weights of layers, and thus, do not provide any insight into the proper bit-width of input activations.
As a result, sensitivity-based quantization approaches often fall short of achieving high model size compression ratios and/or maintaining accuracy levels, leaving ample space for further optimization.

In reinforcement learning (RL)-based mixed-precision quantization, an agent interacts with a quantized DNN environment, adjusting the bit-width configurations for various layers using RL algorithms, such as policy gradients, which provide rewards based on model accuracy and resource consumption, enabling the agent to discover optimal quantization strategies through iterative exploration and exploitation \cite{DBLP:conf/cvpr/WangLLLH19,DBLP:journals/micro/ElthakebPMYE20,DBLP:conf/iclr/LouGKLJ20}.
Despite their potential benefits, these RL-based techniques are confronted with a significant challenge: the considerable search time involved in the RL training process.
As a consequence, achieving favorable results within specified GPU-hour constraints becomes challenging.

Differentiable search-based approaches build upon the concept of differentiable architecture search \cite{DBLP:conf/iclr/LiuSY19} and apply it to mixed-precision quantization.
In these approaches, a super-network is trained to replace each layer (or activation function) of a DNN, such as a convolutional layer, with parallel branches where each branch implements a quantized version of the layer (or activation function) \cite{DBLP:conf/cvpr/CaiV20}.
These approaches have significant drawbacks.
First, they suffer from large training times and high GPU RAM requirements due to the very large size of the super-network.
Second, the distribution defined over each group of parallel branches, which is found during training, may not converge to a uni-modal distribution, rendering the selection of a single quantized version of a layer infeasible.

Lastly, sequential model-based mixed-precision quantization approaches employ surrogate models that define a mapping between the search space and an objective function to guide the exploration of the search space \cite{DBLP:conf/date/SonPVC23}.
Tree-structured Parzen estimator (TPE) is a powerful tool that has shown great success in hyperparameter tuning \cite{DBLP:conf/icml/BergstraYC13}.
However, its naive application to mixed-precision quantization ignores the flat loss landscapes that are prevalent in DNNs.
This greatly increases the search time and yields inferior objective values.
Our TPE-based optimization achieves a \(12\times\) average search time speedup and a 20\% reduction in model size compared to \cite{DBLP:conf/date/SonPVC23} while preserving the model accuracy.

\section{Proposed Method}
\label{sec:proposed-method}
\noindent
This section details the three main components of our model-based bit-width and layer-width optimization framework, i.e., Hessian-based search space pruning, \(K\)-Means TPE, and hardware-aware performance modeling.

\subsection{Hessian-Based Search Space Pruning} 
\noindent
The size of the search space grows exponentially as the number of DNN layers increases.
As a result, pruning the search space by eliminating the bit-width choices that are likely to hamper the model accuracy or cause unnecessary computations is of paramount importance due to its exponential reduction in the size of the search space.
The Hessian of the loss function with respect to each layer's weights provides an excellent starting point to evaluate the criticality of the bit-width for each layer. We first prove an important result. 

\begin{lem}
    The maximum error induced in a DNN's output by unit perturbation in a layer's parameters is bounded by the trace value of the Hessian matrix of the loss with respect to that layer's parameters.
\end{lem}

\begin{proof}
    Assume we freeze all parameters of a DNN except those of a single convolutional filter or neuron in layer \(l\).
    Let \(\mathbold{w}_l\) and \(\mathbold{w}_l^\mathrm{q}\) denote the parameters of layer \(l\) and their corresponding quantized values, respectively.
    Taylor’s Theorem implies that the output loss may be approximated around $\mathbold{w}_l^\mathrm{q}$ as follows: 
    \begin{equation*}
    \begin{split}
        \mathcal{L}(\mathbold{w}_l) \approx \mathcal{L}(\mathbold{w}_l^\mathrm{q}) & + (\mathbold{w}_l - \mathbold{w}_l^\mathrm{q})^\mathrm{T} \mathbold{\nabla}\mathcal{L}_{\mathbold{w}_l}\\
        & + \frac{1}{2}(\mathbold{w}_l - \mathbold{w}_l^\mathrm{q})^\mathrm{T} \mathbold{H}_{\mathbold{w}_l}(\mathbold{w}_l - \mathbold{w}_l^\mathrm{q}),
    \end{split}
    \end{equation*}
    where \(\mathbold{\nabla}\mathcal{L}_{\mathbold{w}_l}\) is the gradient vector of the loss function with respect to parameters \(\mathbold{w}\) evaluated at \(\mathbold{w}_l\), which is nearly zero on a trained model, and \(\mathbold{H}_{\mathbold{w}_l}\) denotes the Hessian matrix, whose entries are the second derivatives of the loss function with respect to parameters \(\mathbold{w}\) evaluated at  \(\mathbold{w}_l^\mathrm{q}\).
    Therefore, we have:
    \begin{equation}
        \label{eq:taylor-expansion-simplified}
        \mathcal{L}(\mathbold{w}_l) - \mathcal{L}(\mathbold{w}_l^\mathrm{q}) \approx \frac{1}{2}(\mathbold{w}_l - \mathbold{w}_l^\mathrm{q})^\mathrm{T} \mathbold{H}_{\mathbold{w}_l}(\mathbold{w}_l - \mathbold{w}_l^\mathrm{q}).
    \end{equation}
    By writing the spectral decomposition of the Hessian matrix as \(\mathbold{H}_{\mathbold{w}_l} = \mathbold{U} \mathbold{D} \mathbold{U}^T\), where \(\mathbold{U}\) is a unitary matrix and \(\mathbold{D}\) is a diagonal matrix, and plugging this expression back into \eqref{eq:taylor-expansion-simplified}, we obtain:
    \begin{equation*}
    \Delta \mathcal{L}(\mathbold{w}_l) \approx \frac{1}{2}(\mathbold{U}^\mathrm{T} \Delta \mathbold{w}_l)^\mathrm{T} \mathbold{D} (\mathbold{U}^\mathrm{T} \Delta \mathbold{w}_l).
    \end{equation*}
   with \(\Delta \mathbold{w}_l = \mathbold{w}_l - \mathbold{w}_l^\mathrm{q}\). Denoting \(\mathbold{U}^T \Delta \mathbold{w}_l\) with \(\mathbold{a}\), we have:
    \begin{equation*}
    \begin{split}
        \Delta \mathcal{L}(\mathbold{w}_l) \approx \frac{1}{2} \mathbold{a}^\mathrm{T} \mathbold{D} \mathbold{a} &= \frac{1}{2} \sum \limits^{}_{i} a_i^2 \lambda_i\\
        &\leq \frac{1}{2} \max_{i}(a_i^2) \sum \limits^{}_{i} \lambda_i.
    \end{split}
    \end{equation*}
    Since \(\mathbold{U}\) is a unitary matrix and the lemma assumes unit perturbation, \(\max_{i}(a_i^2) \leq 1\).
    Therefore, 
    \begin{equation*}
    \Delta \mathcal{L}(\mathbold{w}_l) \leq \frac{1}{2} \mathrm{Tr}(\mathbold{H}_{\mathbold{w}_l}).
    \end{equation*}
\end{proof}
The significance of this lemma lies in its implications for training and quantizing DNNs. For example, a small trace value of the Hessian matrix indicates that the loss function's surface is relatively flat around the current parameter values, making the model less sensitive to perturbations and potentially more stable during training. On the other hand, a large trace value suggests that the loss function's surface has steeper curves, indicating higher sensitivity to parameter changes and potentially making training more challenging.

The Hessian-based search space pruning algorithm starts by normalizing the trace of the Hessian of the loss function with respect to each layer's weight by the number of weights in that layer to find an estimate of the relative importance of layers in a DNN.
It then applies \(k\)-means clustering to the normalized trace values, sorts clusters in non-increasing order of their centroid values, and assigns candidate bit-widths to the layers within each cluster according to these centroid values by assigning higher bit-widths to layers that are within clusters with larger trace values. We use the same bit-width for weights and input activations of a DNN layer for improved hardware performance.
For example, when \(k = 4\), possibly overlapping subsets of candidate bit-widths \(B = \{8, 6, 4, 3, 2\}\) can be considered for different clusters, e.g., \(B_1 = \{8, 6\}\), \(B_2 = \{6, 4, 3\}\), \(B_3 = \{4, 3, 2\}\), and \(B_4 = \{3, 2\}\), where the bit-widths of layers within the cluster with the largest centroid value are selected from \(B_1\), and those for the second largest centroid value are chosen from \(B_2\), and so on.\footnote{The part of the search space defined by the layer-width values is not pruned, and layer-width is always taken from the set \(S = \{0.75, 0.875, 1, 1.125, 1.25 \}\).}

\subsection{\(K\)-Means TPE}
\noindent
An integral part of the presented framework is a novel sequential model-based optimization methodology that is based on the TPE methods.
The main idea behind TPE methods is to recursively partition the data space into smaller regions (nodes) and estimate the probability density function within each region separately. 
In particular, the TPE methods use Bayesian reasoning to propose configurations from the search space that are likely to improve an objective function as explained next.
After drawing a few random configurations \(\chi = \{x^{(1)}, x^{(2)}, \ldots, x^{(k)}\}\) from the search space and observing their corresponding objective function values \(Y = \{y^{(1)}, y^{(2)}, \ldots, y^{(k)} \}\), the TPE methods define a threshold \(\hat y\) that is equal to the largest of \(q\)-quantiles of \(Y\), and creates a surrogate distribution \(l(x)\) for configurations with desirable objective function values (\(y^{(i)} \geq \hat y\) in a maximization problem) and a surrogate distribution \(g(x)\) for configurations with undesirable objective function values (\(y^{(i)} < \hat y\).)
The candidate configuration to be evaluated next is an \(\tilde{x}\) that maximizes the \(\frac{l(x)}{g(x)}\) ratio.
The chosen \(\tilde{y}\) leads to an update in \(\hat y\), \(l(x)\), and \(g(x)\) until a stopping criterion is met. 
%
%

Due to the flat loss landscapes that are prevalent in DNNs, widely different configurations from the search space may yield very close objective values.
This becomes problematic when objective values of configurations from promising parts of the search space fall slightly below \(\hat y\), which puts those configurations in \(g(x)\).
This effectively discourages exploring those parts of the search space, which, in turn, can yield configurations with inferior objective values.

To address this problem, we introduce a novel dual-threshold TPE method that incorporates \(k\)-means clustering in the threshold definition process (which we name \(k\)-means TPE.)
The dual-threshold optimizer starts with an initial \(k\) (where \(k \geq 3\)) for clustering the elements of \(Y\), sorts clusters in decreasing order of their centroid values (\(C_1, \ldots, C_k\)), and defines the surrogate distributions as follows:
\begin{equation*}
    p(x|y)=
    \begin{cases}
        l(x), & \text{if \(y \in C_1\) } \\
        g(x), & \text{if \(y \in C_k\) .}
    \end{cases}
\end{equation*}
After every few iterations of the search process, we increase $k$, which tightens the criteria of being a desirable or undesirable configuration.
This effectively implements an annealing process that initially allows for large moves in the search space to explore distant regions of the search space, and gradually reduces the move size so that search is narrowed to nearby regions of the search space (i.e., those regions that are close to the currently found promising solutions.)

\subsection{Hardware-Aware Objective Function}
\label{subsec:hardware-aware-objective}
\noindent
The problem addressed in this work is mathematically formulated in its general form as follows:
\begin{equation*}
\begin{aligned}
    \max_{B, S} \quad & \mathrm{accuracy}(\mathbold{\Theta}, \mathcal{D}, B, S)\\
    \textrm{s.t.} \quad & \mathrm{modelSize}(\mathbold{\Theta}, B,S) \leq \mu\\ 
    & \mathrm{latency}(\mathbold{\Theta}, \mathcal{H}, B,S) \leq \tau \\
    & \mathrm{energy}(\mathbold{\Theta}, \mathcal{H}, B,S) \leq \epsilon \\
    & \mathrm{throughput}(\mathbold{\Theta}, \mathcal{H}, B,S) \geq \pi,
\end{aligned}
\end{equation*}
where \(\mathbold{\Theta}\) denotes the parameters of the target DNN to be optimized, \(\mathcal{D}\) comprises samples from a target dataset, \(B\) and \(S\) denote the sets of candidate bit-widths and layer-widths multipliers, respectively, \(\mathcal{H}\) characterizes the target hardware, \(\mu\), \(\tau\), and \(\epsilon\) represent the model size, latency, and energy consumption upper bounds, and \(\pi\) denotes the throughput lower bound.

The optimization problem can be addressed by solving its Lagrangian dual problem, where large Lagrangian multipliers are assigned to various constraints.
In practical applications, one or more of the said constraints may be relaxed.
The focus of this work is on the model size and latency constraints.

Exemplary target hardware used in this work is a Xilinx FPGA in which the chip layout includes columns of DSPs, BRAMs, and CLBs.
Our design for the hardware architecture of the accelerator that processes a given DNN is comprised of
\begin{itemize}
    \item a 2D systolic array \(M \times N\) of processing elements (PEs) where each PE contains one DSP and a companion BRAM, and
    \item a memory hierarchy that encompasses off-chip DRAM, on-chip URAMs and BRAMs, and register files.
\end{itemize}
When processing some DNN layer, a first set of \(N\) input activations are loaded into the first row of the systolic array, and each input activation is multiplied with a corresponding weight for the first output filter, residing in the BRAM associated with each of the PEs in that row.
In the next cycle, the first set of input activations is passed down the second row of PEs, while a second set of input activations are loaded into the first row. 
The second set of input activations is multiplied with a second set of weights associated with the first filter while the first set of activations is multiplied with the first set of weights associated with the second output filter. 
Evidently, this process is repeated multiple times (i.e., \(N'/N\) where $N$ is the number of entries in the input feature patch, which is commonly equal to \(3 \times 3 \times I\) with $I$ denoting the input channel count) to produce the filter results for \(M\) output channels of the layer.
The partial products computed in different cycles are accumulated within each PE in the 2D array. 
In the end, all partial accumulation results stored locally in the PEs of each row are passed onto a  tree adder structure (which we call a processing unit, of which there are \(M\)) to produce the final scalar convolution result for each of the output channels. 
We point out that when the number of output channels \(M'\) is larger than \(M\), the systolic array must be invoked \(M'/M\) times.
Each DSP block can perform a \(27 \times 18\)-bit multiplication followed by a 48-bit accumulation.
A crux of our design lies in packing multiple low-bit-width operands in each line of memory in addition to utilizing each DSP to efficiently perform multiple multiplications and additions.

We extend the idea of HiKonv \cite{DBLP:conf/aspdac/LiuCGPXC22}, which introduces packed operations for 1D convolutions, to support 2D convolutions with arbitrary bit-widths.
More specifically, our operand and operation packing approach is capable of performing two multiplications for eight- or six-bit operands, six multiplications and two additions for four- or three-bit operands, and 15 multiplications and eight additions for two-bit operands all while only using a single DSP.
Fig.~\ref{fig:packing} illustrates an example of operand and operation packing for four-bit operands.
\begin{figure}[tb]
    \centering
    \includegraphics[width=\columnwidth]{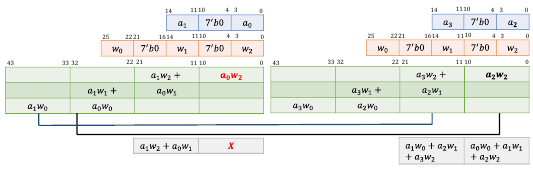}
    \caption{Four-bit operand and operation packing. The design yields the computations required for two rows of convolutional kernels every two cycles.}
    \label{fig:packing}
\end{figure}

As a result of our architecture design and packing scheme, weight quantization yields linear weight size reduction as a function of the bit-width selected for a layer while latency reduction is a function of the number of operations that can be packed as explained above.
Considering the total number of layers and the number of weights/operations per layer, the overall model size reduction and speedup can be easily calculated.
Alg.~\ref{alg:k-means-based-TPE} summarizes different steps of the search process presented in this work. 
\begin{algorithm}[tb]
\footnotesize
  \caption{\(k\)-means-based TPE}
  \label{alg:k-means-based-TPE}
  \begin{algorithmic}[1]
    \Require
        \Statex  \(\mathbold{\Theta}\) \Comment{DNN to be optimized}
        \Statex  \(\mathcal{D}\) \Comment{dataset}
        \Statex  \(\mu\), \(\tau\)       \Comment{model size limit and latency constraint}
        \Statex \(B, S\) \Comment{available bit-widths and layer-width multipliers}
        \Statex \({n_0}\) \Comment{no. of random configurations to evaluate before building surrogates}
        \Statex \({n}\) \Comment{total number of configurations to evaluate}
        \Statex \(\alpha = 0.98, \mathrm{maxiters} = 100, c = 0.25\)
            
    \Ensure
        \Statex \(x^*\) \Comment{the optimal bit-widths and layer-widths configurations}
    
    \State \(\texttt{sensitivities} = \texttt{analyze\_hessian}(\mathbold{\Theta},\mathcal{D})\)
    \State \(\texttt{space} = \texttt{create\_search\_space}(\texttt{sensitivities}, B, S)\) 
    \State \(\chi = \texttt{sample\_randomly}(\texttt{space}, n_0)\)

    \Statex
    
    \State \(\chi_l = \emptyset\)    \Comment{desirable configurations}
    \State \(\chi_g = \emptyset\)    \Comment{undesirable configurations}
    \State \(\texttt{objectives} = []\) \Comment{initializing objective values list}

    \ForEach {\(x\) \(\in \chi\)} 
        \State \({\texttt{objectives.add}(\texttt{evaluate}(x, \mathbold{\Theta}, \mathcal{D}, \mu, \tau))}\)
    \EndFor

    \Do
        \State \(k = \lceil \frac{1}{c} \rceil\)
        \State \(\texttt{clusters} = \texttt{k\_means\_and\_sort}(\texttt{objectives}, k)\)
        \State \(\chi_l = C_1, \chi_g = C_k\)
        \State \(l(x) = \texttt{fit\_gaussian}(\chi_l)\)
        \State \(g(x) = \texttt{fit\_gaussian}(\chi_g)\)
        \State \(\tilde{x} = \argmax_x l(x)/g(x)\)
        \State \(\chi = \chi \bigcup \{\tilde{x}\}\)
        \State \({\texttt{objectives.add}(\texttt{evaluate}(\tilde{x}, \mathbold{\Theta}, \mathcal{D}, \mu, \tau))}\)
        \State \(c = c \times \alpha\)
    \doWhile {\(\mathrm{iter} \leq (n - n_0)\)}
    
    \State \(x^* = \max_x \texttt{objectives}\)
    \State{\textbf{return} \(x^*\)}

\end{algorithmic}
\end{algorithm}

\section{Results \& Discussion}
\label{sec:results}
\noindent
This section presents the results of experiments that evaluate the effectiveness of our methodology, which involve hyperparameter tuning in addition to mixed-precision quantization and layer-width scaling through neural architecture search.

\subsection{Convergence of \(K\)-Means TPE}

This section presents the speedup in convergence for three types of machine learning models and three datasets.
The first experiment involves applying random forest regression to the Iris dataset.
The variables that define the search space are the number of trees in the forest, the maximum depth of each tree, and the minimum number of samples required to split an internal node.

The second experiment involves training a gradient boosting classifer on the Titanic dataset.
The variables that define the search space are the learning rate, number of boosting stages, maximum depth of the estimator, minimum number of samples required to split, minimum number of samples needed to be at a leaf node, and number of features to consider when looking for the best split.
For the first two experiments, \(n_0 = 20\), \(n = 100\), \(k = 4\), and \(\alpha = 0.98\).

Finally, the third experiment involves mixed-precision quantization and layer-width scaling of a ResNet-18 model on the CIFAR-100 dataset.
For this experiment, \(n_0 = 40\), \(n = 160\), \(k = 4\), and \(\alpha = 0.98\).
Figure \ref{fig:convergence} showcases the superior convergence speed of the presented \(k\)-means TPE over the traditional TPE for all experiments.
\begin{figure}[tb]
    \centering
    \includegraphics[width=\columnwidth]{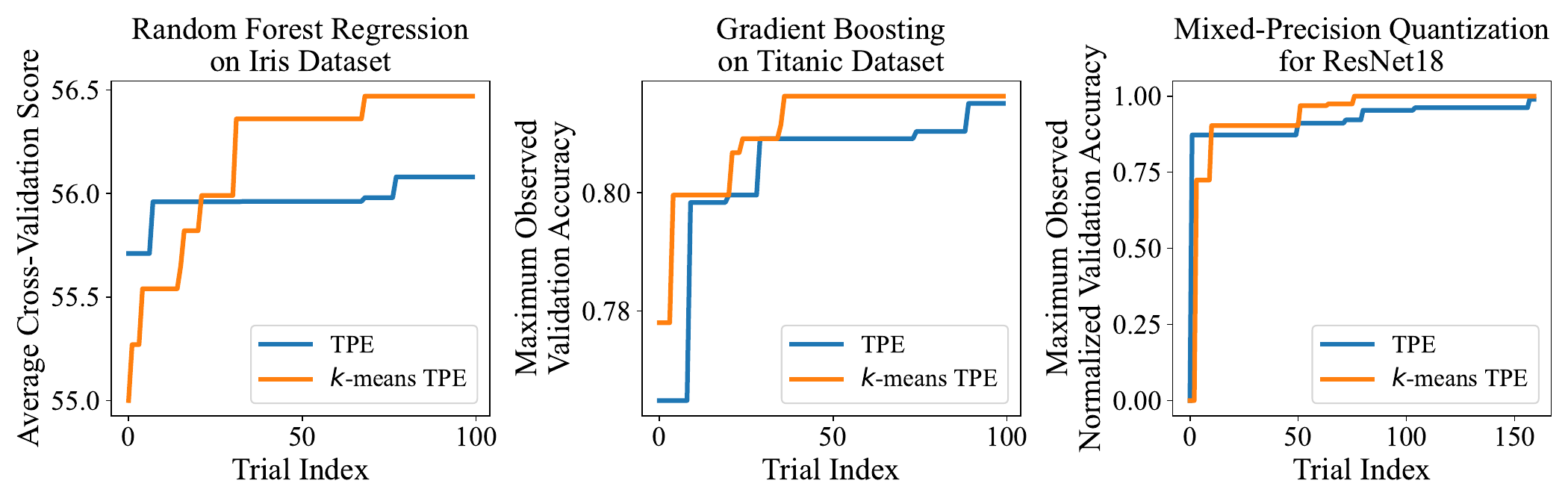}
    \caption[Convergence comparison]{Comparison of the convergence speed of TPE and \(k\)-means TPE for different machine learning algorithms and on Iris, Titanic, and CIFAR-100 datasets.}
    \label{fig:convergence}
\end{figure}
It is observed that \(k\)-means TPE converges to superior or same-quality results at about two to three times fewer evaluations of proposed configurations compared to TPE.

\subsection{Mixed-Precision Quantization and Layer-Width Scaling}

To show the effectiveness of \(k\)-means TPE on a variety of DNNs and datasets, we integrate it into the HyperOpt \cite{DBLP:conf/icml/BergstraYC13} library, and use the PyTorch library \cite{DBLP:conf/nips/PaszkeGMLBCKLGA19} to train and quantize the models.
For all experiments, we only use a very small number of epochs to evaluate different configurations during the search process, similar to \cite{DBLP:conf/iclr/LiuSY19}.
More specifically, for models trained on the CIFAR-10 and CIFAR-100 datasets, the number of epochs to train is set to four, and for the more extensive ImageNet dataset \cite{DBLP:conf/cvpr/DengDSLL009}, this number is set to one.
As shown in Table~\ref{tab:epochs_per_config} for an exemplary model and dataset, such an approximation does not make a noticeable difference in the final results compared to a scenario where each configuration is trained for a much larger number of epochs.
\begin{table}[tb]
\centering
\caption{Impact of the number of epochs to train each configuration on final accuracy, model size, and speedup (compared to FiP-16)}
\label{tab:epochs_per_config}
\resizebox{\columnwidth}{!}{
\begin{tabular}{c c c c c c}
    \toprule
    \multirow{2}{*}{\textbf{Dataset}}   & \multirow{2}{*}{\textbf{Architecture}}    & \textbf{Approach} & \textbf{Accuracy}         & \textbf{Model}         &  \multirow{2}{*}{\textbf{Speedup}} \\
    {}                                  & {}                                        & \textbf{(Epochs per Config.)}  & \textbf{(\%)}             & \textbf{Size (MB)}     & {} \\
    \midrule[\heavyrulewidth]
    \multirow{2}{*}{\textbf{CIFAR-10}}    & \multirow{2}{*}{\textbf{ResNet-20}}    &
    90    & {91.94}    & {0.097}    & {10.86\(\times\)} \\
    {} & {} & 4    & {91.90}    & {0.088}    & {11.14\(\times\)} \\
    \bottomrule
\end{tabular}}
\end{table}
After finding the best configuration, we train the models for 90 epochs using the Adam \cite{DBLP:journals/corr/KingmaB14} optimizer with a weight decay equal to \(10^{-4}\).
We apply OneCycleLR learning rate scheduling with maximum learning rate of 0.01.

Table~\ref{tab:results} compares the accuracy, model size, and speedup (in terms of latency) of different DNNs trained on different datasets and quantized using a variety of approaches.
\begin{table}[tb]
\centering
\caption{Comparison of accuracy, model size, and speedup on different datasets and DNN architectures}
\label{tab:results}
\resizebox{\columnwidth}{!}{
\begin{tabular}{c c l c c c}

\toprule
\multirow{2}{*}{\textbf{Dataset}}   & \multirow{2}{*}{\textbf{Architecture}}    & \textbf{Approach} & \textbf{Accuracy}         & \textbf{Model}         &  \multirow{2}{*}{\textbf{Speedup}} \\
{}                                  & {}                                        & {\textbf{(W/A)}}  & \textbf{(\%)}             & \textbf{Size (MB)}     & {} \\
\midrule[\heavyrulewidth]
\multirow{23}{*}{\rotatebox[origin=c]{90}{\textbf{ImageNet}}}    & \multirow{6}{*}{\textbf{ResNet-18}}    &
Baseline (FiP16/FiP16)    & {71.0}    & {23.38}    & {1.00\(\times\)} \\
\cmidrule{3-6}
{} & {} & {PACT \cite{DBLP:journals/corr/abs-1805-06085} (3/3)}    & {68.1}    & {4.38}    & {-} \\
{} & {} & {AutoQ \cite{DBLP:conf/iclr/LouGKLJ20} (4MP/4MP)} & {68.2} & {5.80} & {-} \\
{} & {} & {\cite{DBLP:conf/eccv/TangOWZJWZ22} (3MP/3MP)}    & {69.7}    & {4.38}    & {-} \\

{}  &  {}  &  {EvoQ \cite{DBLP:conf/ijcnn/Yuan0HP20} (?MP/?MP)}  &  {68.5}  &  {5.85} & {-} \\
{}  & {}    & {Ours (2MP/2MP)}    & {\textbf{70.8}}    & {\textbf{4.01}}    & {10.90\(\times\)} \\

\cmidrule[\heavyrulewidth]{2-6}
{}  & \multirow{7}{*}{\textbf{MobileNetV2}}    &
Baseline (FiP16/FiP16)    & {72.6}    & {6.8}    & {1.00\(\times\)} \\
\cmidrule{3-6}
{} & {} & {PACT \cite{DBLP:journals/corr/abs-1805-06085} (4/4)}    & {61.4}    & {-}    & {5.09\(\times\)} \\
{}  & {}    & {HAQ \cite{DBLP:conf/cvpr/WangLLLH19} (4MP/4MP)}    & {67.0}    & {-}    & {5.13\(\times\)} \\
{}  &  {}  &  {\cite{DBLP:journals/corr/HanMD15} (4/4)}  &  {71.2}  & {1.79}   &  {-} \\
{}  & {}  &  {EMQ \cite{DBLP:conf/icassp/LiuZWM021}}  &  {71.0}  &  {\textbf{1.45}}  &  {-} \\
{}  &  {}  &  {EvoQ \cite{DBLP:conf/ijcnn/Yuan0HP20} (?MP/?MP)}  &  {68.9}  &  {1.78} & {-} \\
{}  & {}    & {Ours (2MP/2MP)}    & {\textbf{72.0}}    & {1.50}    & {7.74\(\times\)} \\

\cmidrule[\heavyrulewidth]{2-6}
{}  & \multirow{10}{*}{\textbf{ResNet-50}}    & 
Baseline (FiP16/FiP16)    & {77.3}    & {51.3}    & {1.00\(\times\)} \\
\cmidrule{3-6}
{} & {} & {PACT \cite{DBLP:journals/corr/abs-1805-06085} (3/3)}    & {75.3}    & {9.17}    & {-} \\
{}  &  {}  &  {\cite{DBLP:journals/corr/HanMD15} (3/3)}  &  {75.1}  & {9.36}  &  {-} \\
{}  & {}    & {HAQ \cite{DBLP:conf/cvpr/WangLLLH19} (3MP/3MP)}    & {75.3}    & {9.22}    & {-} \\
{}  &  {}  &  {HAWQ \cite{DBLP:conf/iccv/DongYGMK19} (?MP/?MP)}  &  {75.5}  & {7.96}   &  {-} \\
{}  &  {}  &  {HAWQ-V2 \cite{DBLP:conf/nips/DongYAGMK20} (?MP/?MP)}  &  {75.8}  &  {7.99}  &  {-} \\
{}  & {}  &  {EMQ \cite{DBLP:conf/icassp/LiuZWM021} (?MP/?MP)}  &  {76.0}  &  {8.26}  &  {-} \\
{}  &  {}  &  {EvoQ \cite{DBLP:conf/ijcnn/Yuan0HP20} (?MP/?MP)}  &  {75.5}  &  {12.81} & {-} \\
{}  &  {}  &  {\cite{DBLP:conf/eccv/TangOWZJWZ22} (3MP/4MP)}   &  {76.9}  &  {7.97}  &  {-} \\
{}  & {}    & {Ours (2MP/2MP)}    & {76.7}    & {\textbf{7.15}}    & {13.56\(\times\)} \\

\midrule[\heavyrulewidth]
\multirow{5}{*}{\rotatebox[origin=c]{90}{\textbf{CIFAR-100}}}  & \multirow{2}{*}{\textbf{ResNet-18}}  & 
Baseline (FiP16/FiP16)    & {76.1}    & {23.38}    & {1.00\(\times\)} \\
\cmidrule{3-6}
{} & {} &{Ours (2MP/2MP)}    & {76.1}    & {2.09}    & {22.55\(\times\)}   \\

\cmidrule[\heavyrulewidth]{2-6}
 & \multirow{3}{*}{\textbf{MobileNetv1}}  &
 Baseline (FiP16/FiP16)    & {65.5}    & {8.4}    & {1.00\(\times\)} \\
\cmidrule{3-6}
 
{} & {} & {Ours (2MP/2MP)}    & {66.09}    & {1.66}    & {11.09\(\times\)} \\

\midrule[\heavyrulewidth]
\multirow{4}{*}{\rotatebox[origin=c]{90}{\textbf{CIFAR-10}}}   & \multirow{4}{*}{\textbf{ResNet-20}}  & 
Baseline (FiP16/FiP16)    & {91.5}    & {0.54}    & {1.00\(\times\)} \\
\cmidrule{3-6}
{} & {} & {ReLeQ \cite{DBLP:journals/micro/ElthakebPMYE20}}    & {91.38}    & {0.101}    & {-} \\
{}  & {}    & {Ours}    & {\textbf{91.9}}    & {\textbf{0.088}}    & {11.14\(\times\)} \\

{} & {}  & {}  & {}  &  {}  &  {} \\
\bottomrule
\end{tabular}}
\end{table}

\subsubsection{ImageNet}

\textbf{ResNet-18 \cite{DBLP:conf/cvpr/HeZRS16}} The ResNet-18 model trained with bit-widths and layer-widths returned by the search process achieves 70.8\% validation accuracy with model size reduced to 4.01MB and latency improved by a factor of 10.9 compared to a 16-bit fixed point baseline.
The trained model outperforms PACT \cite{DBLP:journals/corr/abs-1805-06085}, AutoQ \cite{DBLP:conf/iclr/LouGKLJ20}, EvoQ \cite{DBLP:conf/ijcnn/Yuan0HP20}, and the mixed-precision quantization presented in \cite{DBLP:conf/eccv/TangOWZJWZ22} in terms of both accuracy and compression (latency numbers are not reported in these references). 
Compared to the best prior work on ResNet-18 \cite{DBLP:conf/eccv/TangOWZJWZ22}, our ResNet-18 achieves 1.1\% higher accuracy at about 9\% smaller model size (please note that \cite{DBLP:conf/eccv/TangOWZJWZ22} uses mixed-precision quantization with a minimum of three bits per weight, but it does not report model size, so we underestimate their model size by assuming all their weights are mapped to three bits).

\textbf{MobileNetV2 \cite{DBLP:conf/cvpr/SandlerHZZC18}} The \(k\)-means TPE can compress the MobileNetV2 model to only 1.5 MB while causing a 0.6\% accuracy drop compared to the baseline floating point model.
These results advance state-of-the-art on the quantization of MobileNetV2:
the optimized MobileNetV2 model has an only 3\% higher model size compared to EMQ \cite{DBLP:conf/icassp/LiuZWM021} while it achieves 1\% higher accuracy on the difficult ImageNet dataset.

\textbf{ResNet-50 \cite{DBLP:conf/cvpr/HeZRS16}} The optimized ResNet-50 has a 7.15 MB model size while it only causes a 0.6\% accuracy degradation.
To the best of our knowledge, our method is the first to compress the network to this level with an acceptable accuracy drop.
The work in \cite{DBLP:conf/eccv/TangOWZJWZ22} achieves a similar accuracy drop but obtains a model with 11.4\% larger model size.

\subsubsection{CIFAR-100}
Using \(k\)-means TPE, we successfully compressed ResNet-18 and MobileNetV1 on CIFAR-100 datasets by factors of 11.18\(\times\) and 5.06\(\times\), respectively, without any loss in accuracy. The inference latency also experienced a significant improvement, as summarized in Table \ref{tab:results}. Figure \ref{fig:ResNet-18samples} showcases some of the samples explored by the search engine for ResNet-18, alongside the best configuration returned by the algorithm. The effective performance of \(k\)-means TPE in discovering high-performing samples is evident.
\begin{figure}[tb]
    \centering
    \includegraphics[width=0.6\columnwidth]{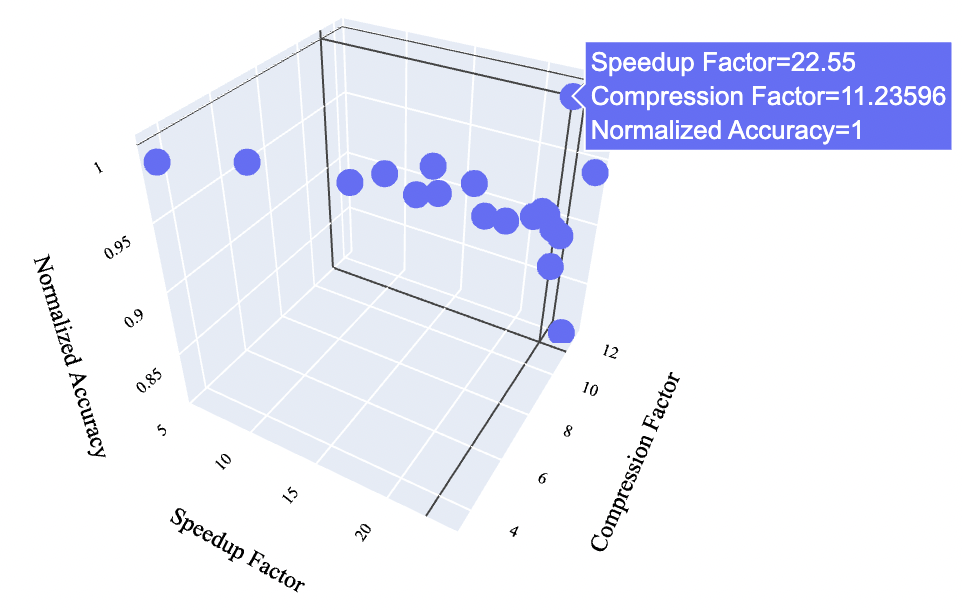}
    \caption[Parallel tree-based reduction]{The space for ResNet-18 compression and the output model}
    \label{fig:ResNet-18samples}
\end{figure}

\subsubsection{CIFAR-10}
For the CIFAR-10 dataset, we conducted experiments with ResNet-20 and compared our results to those of ReLeQ \cite{DBLP:journals/micro/ElthakebPMYE20}. As demonstrated in Table \ref{tab:results}, \(k\)-means TPE outperformed ReLeQ in terms of compressing to ultralow storage size while preserving accuracy levels. Our approach showcased its superiority in achieving highly efficient models without compromising performance.

To gain deeper insights into the efficiency of our \(k\)-means TPE search, we conducted a comparison with the recent work BOMP-NAS \cite{DBLP:conf/date/SonPVC23}, which achieved state-of-the-art results in terms of GPU-hours of search and compression level. As depicted in Table \ref{tab:comparison_to_bomp}, our ResNet-20 model achieved nearly the same level of accuracy as BOMP-NAS, while being 31.5\% smaller in size. Furthermore, the search cost for \(k\)-means TPE was 9.23\(\times\) less than that of BOMP-NAS. Additionally, our ResNet-18 model was 40\% more compressed than BOMP-NAS, with a search time that was 14.63\(\times\) less than the reported GPU-hours for BOMP-NAS. These results highlight the significant advantages of our \(k\)-means TPE approach in terms of search efficiency and model compression.

\begin{table}[tb]
\centering
\caption{Comparison with BOMP-NAS}
\label{tab:comparison_to_bomp}
\resizebox{\columnwidth}{!}{
\begin{tabular}{c l c c c c}

\toprule
\multirow{2}{*}{\textbf{Dataset}}   & \multirow{2}{*}{\textbf{Approach}}    & \textbf{Accuracy} & \textbf{Model Size}         & \textbf{Speedup}         &  \textbf{Search Cost} \\
{}                                  & {}                                        & \textbf{(\%)}  & \textbf{(MB)}             &  {}     & \textbf{(GPU-hour)} \\

\midrule[\heavyrulewidth]
\multirow{2}{*}{\textbf{CIFAR-10}}  & BOMP-NAS \cite{DBLP:conf/date/SonPVC23}  & 
{\textbf{88.67}}    & {0.076}    & {-}    & {12.00} \\
{} & Our ResNet-20 & {88.60} & {\textbf{0.052}} & {18.8\(\times\)} & {\textbf{1.30}} \\
\midrule[\heavyrulewidth]
\multirow{2}{*}{\textbf{CIFAR-100}}   & BOMP-NAS \cite{DBLP:conf/date/SonPVC23}  & 
{75.84}    & {4.199}    & {-}    & {30.00} \\

{} & Our ResNet-18 & \textbf{76.10}    & \textbf{2.090}    & {\textbf{8.5\(\times\)}}    & \textbf{2.05} \\

\bottomrule
\end{tabular}}
\end{table}

Table \ref{table:platforms} shows some configurations found by \(k\)-means TPE. For each model and dataset, the first and second rows of configurations contain the assigned bit-width for each layer and the assigned layer-width scaling factor. As the table demonstrates, to be able to quantize some layers to ultralow precisions (e.g., 2 or 3 bits), the method may strategically decide to scale up the number of filters in that layer. By doing so, the algorithm effectively mitigates quantization errors, achieving a favorable trade-off between precision reduction and layer-width scaling. This demonstrates the effectiveness of the joint optimization of bit-widths and layer-widths through \(k\)-means TPE. 

\begin{table*}[tb]
\centering
\caption{Configurations returned by \(k\)-means TPE for representative DNN architectures}
\resizebox{\textwidth}{!}{
\begin{tabular}{@{}l l l @{}}
\toprule
Model        & Dataset      & Configuration  \\
\midrule
   \multirow{2}{*}{ResNet-18}     & \multirow{2}{*}{ImageNet}   & {8, 6, 6, 4, 4, 6, 6, 4, 4, 4, 4, 2, 2, 3, 3, 2, 2, 6}   \\
  {} & {} & {1.25, 1.25, 1.25, 1.25, 1.25, 0.875, 0.875, 0.875, 0.875, 1, 1, 1, 1, 1, 1, 1, 1, 1}    \\
  \hline
\multirow{2}{*}{ResNet-20}      & \multirow{2}{*}{CIFAR-10}    & {8, 3, 3, 3, 3, 3, 3, 3, 3, 2, 2, 2, 2, 3, 3, 2, 2, 3, 3, 8}  \\ {} & {} & {1, 1, 1, 1, 1, 1, 1, 0.75, 0.75, 0.75, 0.75, 0.75, 0.75, 0.75, 0.75, 0.75, 0.75, 0.75, 0.75, 0.75}   \\
\hline
\multirow{2}{*}{MobileNetV1} & \multirow{2}{*}{CIFAR-100}  &
 {8, 8, 8, 6, 6, 4, 4, 4, 4, 4, 4, 4, 4, 3, 3, 3, 3, 4, 4, 3, 3, 4, 4, 3, 3, 2, 2, 2} \\ 
{} & {} & {1, 1.25, 1.25, 0.875, 0.875, 0.875, 0.875, 1.125, 1.125, 0.875, 0.875, 1.25, 1.25, 1, 1, 1.125, 1.125, 1.25, 1.25, 1.25, 1.25, 1.125, 1.125, 0.75, 0.75, 1, 1, 0.875}\\
\bottomrule
\end{tabular}}
\label{table:platforms}
\end{table*}

\section{Conclusions}
\label{sec:conclusions}

We presented a search-based approach including a Hessian-based pruner and a tree-structured dual-threshold Parzen estimator for automatic optimization of DNNs bit-width and layer-width configurations to enable their efficient deployment.
Through extensive experiments on benchmark datasets, we showed a \(12\times\) average search time speedup and a 20\% reduction in model size compared to the state-of-the-art compression techniques while preserving the model's output accuracy.
Such an advancement facilitates rapid model development and deployment in resource-constrained environments, unlocking new possibilities for scalable deep learning systems.

\balance
\bibliographystyle{IEEEtran}
\bibliography{IEEEabrv,HWNAS}
	
\end{document}